\documentclass{article} 
\usepackage{nips12submit_e,times}

\usepackage{amsmath,amsthm,amssymb}

\usepackage{color}
\definecolor{Blue}{rgb}{0.9,0.3,0.3}

\newcommand{\squishlist}{
   \begin{list}{$\bullet$}
    { \setlength{\itemsep}{0pt}      \setlength{\parsep}{3pt}
      \setlength{\topsep}{3pt}       \setlength{\partopsep}{0pt}
      \setlength{\leftmargin}{1.5em} \setlength{\labelwidth}{1em}
      \setlength{\labelsep}{0.5em} } }

\newcommand{\squishlisttwo}{
   \begin{list}{$\bullet$}
    { \setlength{\itemsep}{0pt}    \setlength{\parsep}{0pt}
      \setlength{\topsep}{0pt}     \setlength{\partopsep}{0pt}
      \setlength{\leftmargin}{2em} \setlength{\labelwidth}{1.5em}
      \setlength{\labelsep}{0.5em} } }

\newcommand{\squishend}{
    \end{list}  }

\newcommand{\defeq}{:=}
\newcommand{\myvec}[1]{\mathbf{#1}}
\newcommand{\myvecsym}[1]{\boldsymbol{#1}}
\newcommand{\ind}[1]{\mathbb{I}(#1)}

\newcommand{\vmu}{\myvecsym{\mu}}

\newcommand{\vphi}{\myvecsym{\phi}}

\newcommand{\vw}{\myvec{w}}

\newcommand{\vx}{\myvec{x}}

\newcommand{\vy}{\myvec{y}}

\newcommand{\vJ}{\myvec{J}}

\newcommand{\vX}{\myvec{X}}

\newcommand{\calX}{{\cal X}}

\newcommand{\be}{\begin{equation}}
\newcommand{\ee}{\end{equation}}
\newcommand{\bea}{\begin{eqnarray}}
\newcommand{\eea}{\end{eqnarray}}
\newcommand{\beaa}{\begin{eqnarray*}}
\newcommand{\eeaa}{\end{eqnarray*}}

\DeclareMathAlphabet{\mathpzc}{OT1}{pzc}{m}{n}

\usepackage{graphicx} 
\usepackage{algorithm}
\usepackage{algorithmic}
\usepackage{amsfonts}
\usepackage{slashbox}
\usepackage{subfigure}
\usepackage{xcolor}
\usepackage[labelfont=bf]{caption}

\DeclareMathOperator*{\argmax}{arg\,max}

\usepackage{amsthm}
\newtheorem{theorem}{Theorem}
\newtheorem{lemma}[theorem]{Lemma} 
\newtheorem{proposition}[theorem]{Proposition} 
\newtheorem{remark}[theorem]{Remark}
\newtheorem{corollary}[theorem]{Corollary}

\title{Herded Gibbs Sampling}

\author{ Luke Bornn \\
Harvard University  \\
\texttt{\small bornn@stat.harvard.edu} \\
\And
Yutian Chen \\
UC Irvine \\
\texttt{\small yutian.chen@uci.edu} \\
\And
Nando de Freitas \\
UBC \\
\texttt{\small nando@cs.ubc.ca} \\
\And
Mareija Eskelin \\
UBC \\
\texttt{\small mareija@cs.ubc.ca} \\
\And
Jing Fang \\
Facebook \\
\texttt{\small jingf@cs.ubc.ca} \\
\And
Max Welling \\
University of Amsterdam \\
\texttt{\small welling@ics.uci.edu} \\
}

\providecommand{\e}[1]{\ensuremath{\times 10^{#1}}}

\nipsfinalcopy 

\usepackage{lipsum}
\newcommand\blfootnote[1]{
  \begingroup
  \renewcommand\thefootnote{}\footnote{#1}
  \addtocounter{footnote}{-1}
  \endgroup
}

\newtheorem{trueDist}{Definition}
\newtheorem{hGDist}[trueDist]{Definition}

\newtheorem{thm:hGconvergence}[theorem]{Theorem}

\begin{document}

\maketitle

\begin{abstract}
The Gibbs sampler is one of the most popular algorithms for inference in statistical models. In this paper, we introduce a herding variant of this algorithm, called herded Gibbs, that is entirely deterministic. We prove that herded Gibbs has an $O(1/T)$ convergence rate for models with independent variables and for fully connected probabilistic graphical models. Herded Gibbs is shown to outperform Gibbs in the tasks of image denoising with MRFs and named entity recognition with CRFs. However, the convergence for herded Gibbs for sparsely connected probabilistic graphical models is still an open problem.
 \blfootnote{Authors are listed in alphabetical order.}  
\end{abstract}

\section{Introduction}
\label{sec:introduction}

Over the last 60 years, we have witnessed great progress in the design of randomized sampling algorithms; see for example \cite{Liu-01,Doucet-01,Andrieu-03,Robert2004} and the references therein. In contrast, the design of deterministic algorithms for ``sampling'' from distributions is still in its inception \cite{Chen:2010,holroyd2010rotor,Chen:2011, Murray2012}. There are, however, many important reasons for pursuing this line of attack on the problem. From a theoretical perspective, this is a well defined mathematical challenge whose solution might have important consequences. It also brings us closer to reconciling the fact that we typically use pseudo-random number generators to run Monte Carlo algorithms on classical, Von Neumann architecture, computers. Moreover, the theory for some of the recently proposed deterministic sampling algorithms has taught us that they can achieve $O(1/T)$ convergence rates \cite{Chen:2010,holroyd2010rotor}, which are much faster than the standard Monte Carlo rates of $O(1/\sqrt{T})$ for computing ergodic averages. From a practical perspective, the design of deterministic sampling algorithms creates an opportunity for researchers to apply a great body of knowledge on optimization to the problem of sampling; see for example \cite{Bach:2012} for an early example of this.
   
The domain of application of currently existing deterministic sampling algorithms is still very narrow. Importantly, there do not exist deterministic tools for sampling from unnormalized multivariate probability distributions. This is very limiting because the problem of sampling from unnormalized distributions is at the heart of the field of Bayesian inference and the probabilistic programming approach to artificial intelligence \cite{Lunn:2000,Carbonetto05,Milch06,Goodman:2008}. At the same time, despite great progress in Monte Carlo simulation, the celebrated Gibbs sampler continues to be one of the most widely-used algorithms. For, example it is the inference engine behind popular statistics packages \cite{Lunn:2000}, several tools for text analysis \cite{Porteous:2008}, and Boltzmann machines \cite{Ackley1985,Hinton2006}. The popularity of Gibbs stems from its simplicity of implementation and the fact that it is a very generic algorithm. 

Without any doubt, it would be remarkable if we could design generic deterministic Gibbs samplers with fast (theoretical and empirical) rates of convergence. In this paper, we take steps toward achieving this goal by capitalizing on a recent idea for deterministic simulation known as herding.
Herding \cite{Welling:2009,Welling:2009b,Gelfand:2010} is a deterministic procedure for generating samples $\vx \in {\cal X}\subseteq \mathbb{R}^n$, such that the empirical moments $\vmu$ of the data are matched. The herding procedure, at iteration $t$, is as follows:
\bea
\vx^{(t)} & = & \argmax_{\vx \in \mathcal{X}} \langle \vw^{(t-1)} , \vphi(\vx) \rangle \nonumber \\
\vw^{(t)} & = & \vw^{(t-1)} + \vmu - \vphi(\vx^{(t)}),
\label{eq:herding}
\eea
where $\vphi: {\cal X} \rightarrow {\cal H}$ is a feature map (statistic) from $\cal X$ to a Hilbert space $\cal H$ with inner product $\langle \cdot, \cdot \rangle$, $\vw \in {\cal H}$ is the vector of parameters, and $\vmu \in {\cal H}$ is the moment vector (expected value of $\vphi$ over the data) that we want to match. If we choose normalized features by making $\|\vphi(\vx)\|$ constant for all $\vx$, then the update to generate samples $\vx^{(t)}$ for $t=1,2,\ldots,T$ in Equation \ref{eq:herding} is equivalent to minimizing the objective
\be
J(\vx_{1}, \ldots, \vx_{T}) = \left\| \vmu - \frac{1}{T} \sum_{t=1}^{T} \vphi(\vx^{(t)}) \right\|^2,
\ee
where $T$ may have \emph{no prior known value} and $\| \cdot \| = \sqrt{\langle \cdot, \cdot \rangle}$ is the naturally defined norm based upon the inner product of the space $\cal H$ \cite{Chen:2010,Bach:2012}. 

Herding can be used to produce samples from \emph{normalized} probability distributions. This is done as follows. Let $\vmu$ denote a discrete, normalized probability distribution, with $\mu_i \in [0,1]$ and $\sum_{i=1}^n \mu_i =1$. A natural feature in this case is the vector $\vphi(x)$ that has all entries equal to zero, except for the entry at the position indicated by $x$. For instance, if $x=2$ and $n=5$, we have $\vphi(x) = (0,1,0,0,0)^T$. Hence, $\widehat{\vmu}=T^{-1} \sum_{t=1}^{T} \vphi(x^{(t)})$ is an empirical estimate of the distribution. 
In this case, one step of the herding algorithm involves finding the largest component of the weight vector ($i^{\star} = \argmax_{i \in \{1,2,\ldots,n\}} \vw_i^{(t-1)}$), setting $x^{(t)} = i^{\star}$, fixing the $i^{\star}$-entry of $\vphi(x^{(t)})$ to one and all other entries to zero, and updating the weight vector: $\vw^{(t)} =  \vw^{(t-1)} + (\vmu - \vphi(x^{(t)}))$. The output is a set of samples $\{x^{(1)},\ldots,x^{(T)}\}$ for which the empirical estimate $\widehat{\vmu}$ converges on the target distribution $\vmu$ as $O(1/T)$.

The herding method, as described thus far, only applies to normalized distributions or to problems where the objective is not to guarantee that the samples come from the right target, but to ensure that some moments are matched. An interpretation of herding in terms of Bayesian quadrature has been put forward recently by \cite{Huszar:2012}.

In this paper, we will show that it is possible to use herding to generate samples from more complex \emph{unnormalized} probability distributions.  In particular, we introduce a deterministic variant of the popular Gibbs sampling algorithm, which we refer to as \emph{herded Gibbs}. While Gibbs relies on drawing samples from the \emph{full-conditionals} at random, herded Gibbs generates the samples by matching the full-conditionals. That is, one simply applies herding to all the full-conditional distributions. 

The experiments will demonstrate that the new algorithm outperforms Gibbs sampling and mean field methods in the domain of sampling from sparsely connected probabilistic graphical models, such as grid-lattice Markov random fields (MRFs) for image denoising and conditional random fields (CRFs) for natural language processing. 

We advance the theory by proving that the deterministic Gibbs algorithm converges for distributions of independent variables and fully-connected probabilistic graphical models. However, a proof establishing suitable conditions that ensure convergence of herded Gibbs sampling for sparsely connected probabilistic graphical models is still unavailable.

\section{Herded Gibbs Sampling}
\label{sec:herdedGibbs}

For a graph of discrete nodes ${\cal G} = (V,E)$, where the set of nodes are the random variables $V = \{ X_{i}\}_{i=1}^{N}$, $X_{i} \in {\calX}$, let $\pi$ denote the \emph{target distribution} defined on $\cal G$.

Gibbs sampling is one of the most popular methods to draw samples from $\pi$. Gibbs alternates (either systematically or randomly) the sampling of each variable $X_i$ given $\vX_{{\cal N}(i)}=\vx_{{\cal N}(i)}$, where $i$ is the index of the node, and ${\cal N}(i)$ denotes the neighbors of node $i$. That is, Gibbs generates each sample from its full-conditional distribution $p(X_{i}|\vx_{{\cal N}(i)})$.

Herded Gibbs replaces the sampling from full-conditionals with herding at the level of the full-conditionals. That is, it alternates a process of matching the full-conditional distributions $p(X_i=x_i|\mathbf{X}_{{\cal N}(i)})$. To do this,
herded Gibbs defines a set of auxiliary weights $\{w_{i, \vx_{{\cal N}(i)}}\}$ for any value of $X_i=x_i$ and $\vX_{{\cal N}(i)}=\vx_{{\cal N}(i)}$.
For ease of presentation, we assume the domain of $X_i$ is binary, $\mathcal{X}=\{0,1\}$, and we use one weight for every $i$ and assignment to the neighbors $\mathbf{x}_{{\cal N}(i)}$.
Herded Gibbs can be trivially generalized to the multivariate setting by employing weight vectors in $\mathbb{R}^{|{\cal X}|} $ instead of scalars. 

If the binary variable $X_i$ has four binary neighbors $\vX_{{\cal N}(i)}$, we must maintain $2^4=16$ weight vectors. Only the weight vector corresponding to the current instantiation of the neighbors is updated, as illustrated in Algorithm~\ref{alg:herdedGibbs}. The memory complexity of herded Gibbs is exponential in the maximum node degree. Note the algorithm is a deterministic Markov process with state $(\mathbf{X}, \mathbf{W})$. 

\begin{algorithm}[t]
\caption{Herded Gibbs Sampling.}
\label{alg:herdedGibbs}
\begin{algorithmic}
\STATE \textbf{Input:}  $T$.
\STATE Step 1: Set $t=0$. Initialize $\mathbf{X}^{(0)}$ in the support of $\pi$ and $w_{i, \mathbf{x}_{{\cal N}(i)}}^{(0)}$ in $(\pi(X_i=1|\mathbf{x}_{{\cal N}(i)})-1,\pi(X_i=1|\mathbf{x}_{{\cal N}(i)}))$.
\FOR{$t = 1 \to T$}
\STATE Step 2: Pick a node $i$ according to some policy. Denote $w = w_{i, \mathbf{x}_{{\cal N}(i)}^{(t-1)}}^{(t-1)}$.
\STATE Step 3: If $w > 0$, set $x_i^{(t)}=1$, otherwise set $x_i^{(t)}=0$.
\STATE Step 4: Update weight $w_{i, \mathbf{x}_{{\cal N}(i)}^{(t)}}^{(t)} = w_{i, \mathbf{x}_{{\cal N}(i)}^{(t-1)}}^{(t-1)} + \pi(X_i=1|\mathbf{x}_{{\cal N}(i)}^{(t-1)}) -  x_i^{(t)}$
\STATE Step 5: Keep the values of all the other nodes $x_j^{(t)}=x_j^{(t-1)}, \forall j\neq i$ and all the other weights $w_{j, \mathbf{x}_{{\cal N}(j)}}^{(t)} = w_{j, \mathbf{x}_{{\cal N}(j)}}^{(t-1)}, \forall j \neq i \textrm{ or } \mathbf{x}_{{\cal N}(j)} \neq \mathbf{x}_{{\cal N}(i)}^{(t-1)}$.
\ENDFOR
\STATE \textbf{Output:} $\vx^{(1)},\ldots,\vx^{(T)}$
\end{algorithmic}
\end{algorithm}

The initialization in step 1 guarantees that $\mathbf{X}^{(t)}$ always remains in the support of $\pi$. For a deterministic scan policy in step 2, we take the value of variables $\mathbf{x}^{(tN)},t\in\mathbb{N}$ as a sample sequence. Throughout the paper all experiments employ a fixed variable traversal for sample generation. We call one such traversal of the variables a \emph{sweep}. 


\section{Analysis}
\label{sec:analysis}


As herded Gibbs sampling is a deterministic algorithm, there is no stationary probability distribution of states.  Instead, we examine the average of the sample states over time and hypothesize that it converges to the joint distribution, our target distribution, $\pi$. To make the treatment precise, we need the following definition:

\begin{hGDist}
For a graph of discrete nodes ${\cal G} = (V,E)$, where the set of nodes $V = \{ X_{i}\}_{i=1}^{N}$, $X_{i} \in {\calX}$, $P_{T}^{(\tau)}$ is the empirical estimate of the joint distribution obtained by averaging over $T$ samples acquired from $\cal G$.
$P_{T}^{(\tau)}$ is derived from $T$ samples, collected at the end of every sweep over $N$ variables, starting from the $\tau^{th}$ sweep:
\bea
P_{T}^{(\tau)} (\mathbf{X} = \mathbf{x}) = \frac{1}{T} \sum_{k = \tau}^{\tau + T -1} \ind{ \mathbf{X}^{(k N)} = \mathbf{x} }
\label{eqn:hGDist}
\eea
\label{def:hGDist}
\end{hGDist}


Our goal is to prove that the limiting average sample distribution over time converges to the target distribution $\pi$. Specifically, we want to show the following: 
\bea
\lim_{T\rightarrow \infty} P_T^{(\tau)} \left( \vx \right) = \pi(\vx), \forall \tau\geq 0
\label{eq:herdingGoal}
\eea
If this holds, we also want to know what the convergence rate is.


We begin the theoretical analysis with a graph of one binary variable. For this graph, there is only one weight $w$. Denote $\pi(X=1)$ as $\pi$ for notational simplicity. The sequence of $X$ is determined by the dynamics of $w$ (shown in Figure~\ref{fig:univariate_1}): 
\bea
w^{(t)} =  w^{(t-1)} + \pi - \ind{w^{(t-1)} > 0},\quad X^{(t)} = \left\{\begin{array}{rl}
1 & \mbox {if $w^{(t-1)} > 0$} \\
0 & \mbox {otherwise}
\end{array}\right.
\label{eq:singleNodeDynamicsA}
\eea
Lemma \ref{lem:single_node} in the appendix shows that $(\pi-1,\pi]$ is the invariant interval of the dynamics, and the state $X=1$ is visited at a frequency close to $\pi$ with an error:
\begin{equation}
|P_T^{(\tau)}(X=1)-\pi|\leq \frac{1}{T}
\end{equation}
This is known as the fast moment matching property in \cite{Welling:2009,Welling:2009b,Gelfand:2010}. We will show in the next two theorems that the fast moment matching property also holds for two special types of graphs, with proofs provided in the appendix.





\begin{figure}[h!]
 \centering
 \includegraphics[width=.46\textwidth] {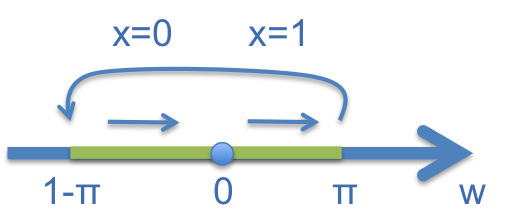}
 \caption{Herding dynamics for a single variable.}
 \label{fig:univariate_1}
\end{figure}

In an empty graph, all the variables are independent of each other and herded Gibbs reduces to running $N$ one-variable chains in parallel. Denote the marginal distribution $\pi_i\defeq\pi(X_i=1)$. 

Examples of failing convergence in the presence of rational ratios between the $\pi_i$s were observed in \cite{Bach:2012}. There the need for further theoretical research on this matter was pointed out. The following theorem provides formal conditions for convergence in the restricted domain of empty graphs. 

\begin{theorem} \label{thm:convergence_empty_graph}
For an empty graph, when herded Gibbs has a fixed scanning order, and $\{1,\pi_1,\dots,\pi_N\}$ are rationally independent, the empirical distribution $P^{(\tau)}_T$ converges to the target distribution $\pi$ as $T\rightarrow \infty$ for any $\tau\geq 0$.
\end{theorem}
A set of $n$ real numbers, $x_1, x_2, \dots, x_n$, is said to be rationally independent if for any set of rational numbers, $a_1, a_2, \dots, a_n$, we have $\sum_{i=1}^n a_i x_i = 0 \Leftrightarrow a_i = 0, \forall 1\leq i \leq n$. The proof of Theorem \ref{thm:convergence_empty_graph} consists of first formulating the dynamics of the weight vector as a constant translation mapping in a circular unit cube, and then proving that the weights are uniformly distributed by making use of Kronecker-Weyl's theorem \cite{weyl}.

For fully-connected (complete) graphs, convergence is guaranteed even with rational ratios. In fact, herded Gibbs converges to the target joint distribution at a rate of $O(1/T)$ with a $O(\log(T))$ burn-in period. This statement is formalized in Theorem~\ref{thm:convergenceHerdedGibbs}. 

\begin{thm:hGconvergence}
For a fully-connected graph, when herded Gibbs has a fixed scanning order and a Dobrushin coefficient of the corresponding Gibbs sampler $\eta<1$, there exist constants $l > 0$, and $B > 0$ such that
\bea
d_{v}(P^{(\tau)}_{T} - \pi) \leq \frac{\lambda}{T}, \forall T \geq T^{*}, \tau > \tau^{*}(T)
\label{eqn:hGconvergence}
\eea
where $\lambda = \frac{2N(1 + \eta)}{l (1 - \eta)}$, $T^* = \frac{2B}{l}$, $\tau^{*}(T) = \log_{\frac{2}{1+\eta}}\left( \frac{(1 - \eta)l T}{4N} \right)$, and $d_{v}(\delta \pi) \defeq \frac{1}{2} || \delta \pi||_{1}$.
\label{thm:convergenceHerdedGibbs}
\end{thm:hGconvergence}


The constants $l$ and $B$ are defined in Equation \ref{eqn:lB} for Proposition \ref{prop:reachability} in the appendix. If we ignore the burn-in period and start collecting samples simply from the beginning, we achieve a convergence rate of $O(\frac{\log(T)}{T})$ as stated in Corollary \ref{cor:logT_T_convergence} in the appendix. The constant $l$ in the convergence rate has an exponential term, with $N$ in the exponent. An exponentially large constant seems to be unavoidable for any sampling algorithm when considering the convergence to a joint distribution with $2^N$ states. As for the marginal distributions, it is obvious that the convergence rate of herded Gibbs is also $O(1/T)$ because marginal probabilities are linear functions of the joint distribution. However, in practice, we observe very rapid convergence results for the marginals, so stronger theoretical results about the convergence of the marginal distributions seem plausible.



The proof proceeds by first bounding the discrepancy between the chain of empirical estimates of the joint obtained by averaging over $T$ herded Gibbs samples, $\{P^{(s)}_{T}\}, s\geq \tau$, and a Gibbs chain initialized at $P^{(\tau)}_{T}$. After one iteration, this discrepancy is bounded above by $O(1/T)$. 

The Gibbs chain has geometric convergence to $\pi$ and the distance between the Gibbs and herded Gibbs chains is bounded by $O(1/T)$.  The geometric convergence rate to $\pi$ dominates the discrepancy of herded Gibbs and thus we infer that $P^{(\tau)}_{T}$ converges to $\pi$ geometrically. To round-off the proof, we must find a limiting value for $\tau$.  The proof concludes with an $O(\log(T))$ burn-in for $\tau$.

However, for a generic graph we have no mathematical guarantees on the convergence rate of herded Gibbs. In fact, one can easily construct synthetic examples for which herded Gibbs does not seem to converge to the true marginals and joint distribution.  For the examples covered by our theorems and for examples with real data, herded Gibbs demonstrates good behaviour. 
The exact conditions under which herded Gibbs converges for sparsely connected graphs are still unknown.

\section{Experiments}
\label{sec:experiments}

\subsection{Simple Complete Graph}
\label{sec:toy}

We begin with an illustration of how herded Gibbs substantially outperforms Gibbs on a simple complete graph. In particular, we consider a fully-connected model of two variables, $X_{1}$ and $X_{2}$, as shown in Figure~\ref{fig:twoNodes}; the joint distribution of these variables is shown in Table~\ref{table:twoNodes}. Figure~\ref{fig:twoNodesMarginals} shows the marginal distribution $P(X_1=1)$ approximated by both Gibbs and herded Gibbs for different $\epsilon$. As $\epsilon$ decreases, both approaches require more iterations to converge, but herded Gibbs clearly outperforms Gibbs. The figure also shows that Herding does indeed exhibit a linear convergence rate.

\begin{minipage}{\textwidth}
\begin{minipage}[b]{0.39\textwidth}
  \centering

\begingroup
  \makeatletter
  \providecommand\color[2][]{%
    \errmessage{(Inkscape) Color is used for the text in Inkscape, but the package 'color.sty' is not loaded}
    \renewcommand\color[2][]{}%
  }
  \providecommand\transparent[1]{%
    \errmessage{(Inkscape) Transparency is used (non-zero) for the text in Inkscape, but the package 'transparent.sty' is not loaded}
    \renewcommand\transparent[1]{}%
  }
  \providecommand\rotatebox[2]{#2}
  \ifx\svgwidth\undefined
    \setlength{\unitlength}{83.56838471pt}
  \else
    \setlength{\unitlength}{\svgwidth}
  \fi
  \global\let\svgwidth\undefined
  \makeatother
  \begin{picture}(1,0.31233802)%
    \put(0,0){\includegraphics[width=\unitlength]{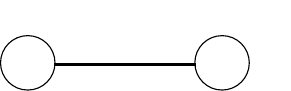}}%
    \put(0.0319586,0.23960567){\color[rgb]{0,0,0}\makebox(0,0)[lb]{\smash{$X_1$}}}%
    \put(0.71164147,0.23413544){\color[rgb]{0,0,0}\makebox(0,0)[lb]{\smash{$X_2$}}}%
  \end{picture}%
\endgroup

  \captionof{figure}{Two-variable model.}
  \label{fig:twoNodes}
\end{minipage}
\hfill
\begin{minipage}[b]{0.59\textwidth}
  \centering
  \begin{tabular}{c|cc|c}
		&$\bf X_1=0$	&$\bf X_1=1$	&$\bf P(X_2)$\\
  \hline
  $\bf X_2=0$	&$1/4-\epsilon$	&$\epsilon$	&$1/4$\\
  $\bf X_2=1$	&$\epsilon$	&$3/4-\epsilon$	&$3/4$\\
  \hline
  $\bf P(X_1)$	&$1/4$		&$3/4$		&$1$
  \end{tabular}
  \captionof{table}{Joint distribution of the two-variable model.}
  \label{table:twoNodes}
\end{minipage}
\end{minipage}

\begin{figure}[h!]
\begin{center}
	\subfigure[Approximate marginals obtained via Gibbs (blue) and herded Gibbs (red).]{
                 \includegraphics[width=\textwidth]{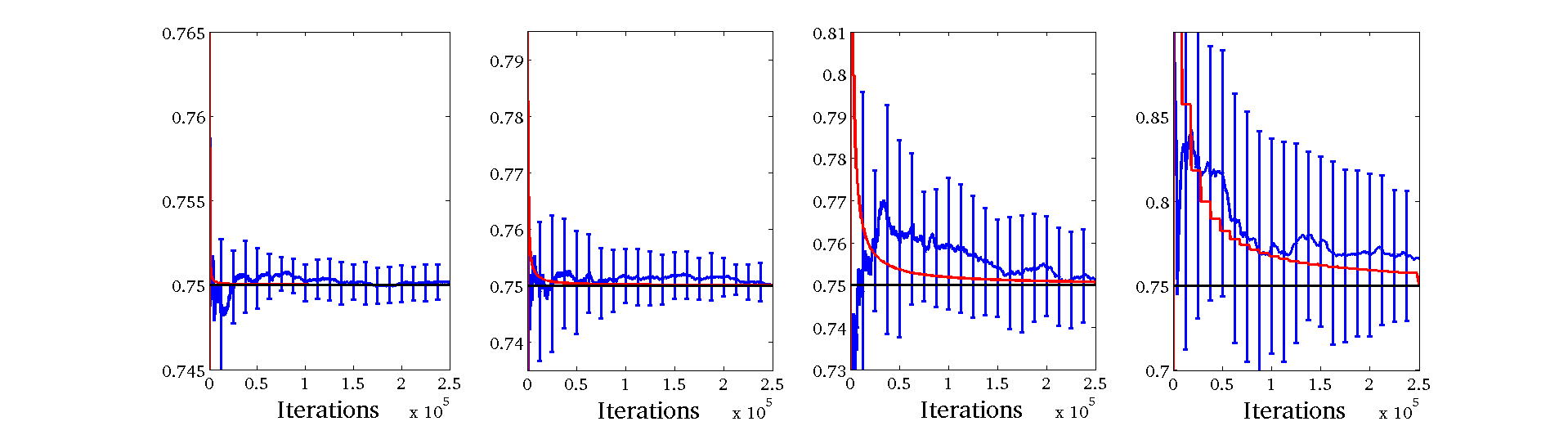}
\label{fig:twoNodesMarginalsSubsA}
           }
\subfigure[Log-log plot of marginal approximation errors obtained via Gibbs (blue) and herded Gibbs (red). ]{
                 \includegraphics[width=\textwidth]{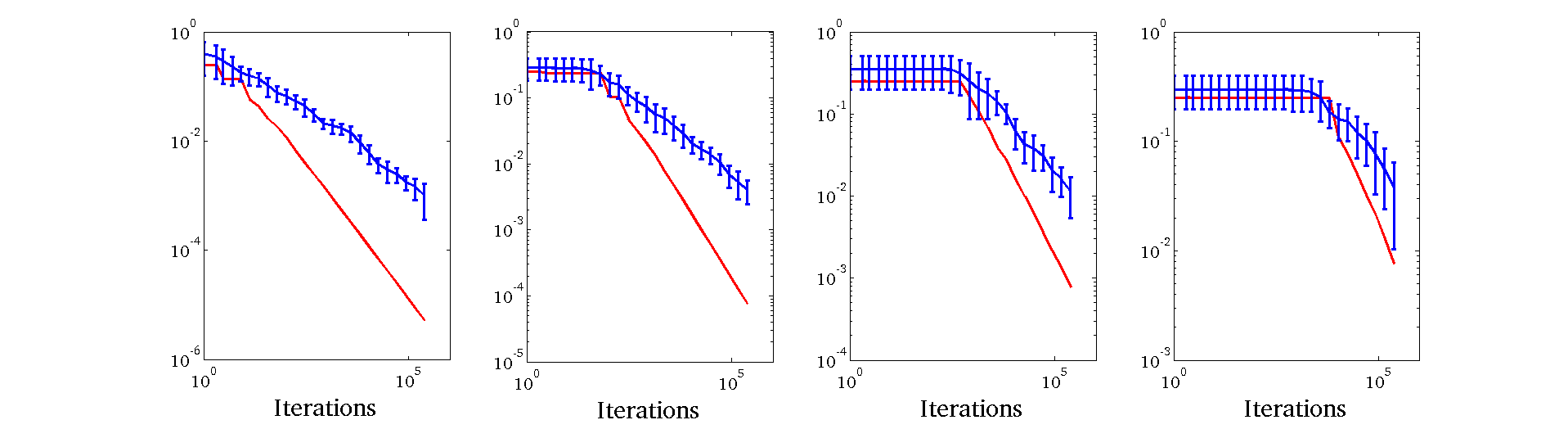}
\label{fig:twoNodesMarginalsSubsB}
           }
\subfigure[Inverse of marginal approximation errors obtained via Gibbs (blue) and herded Gibbs (red).]{
                 \includegraphics[width=\textwidth]{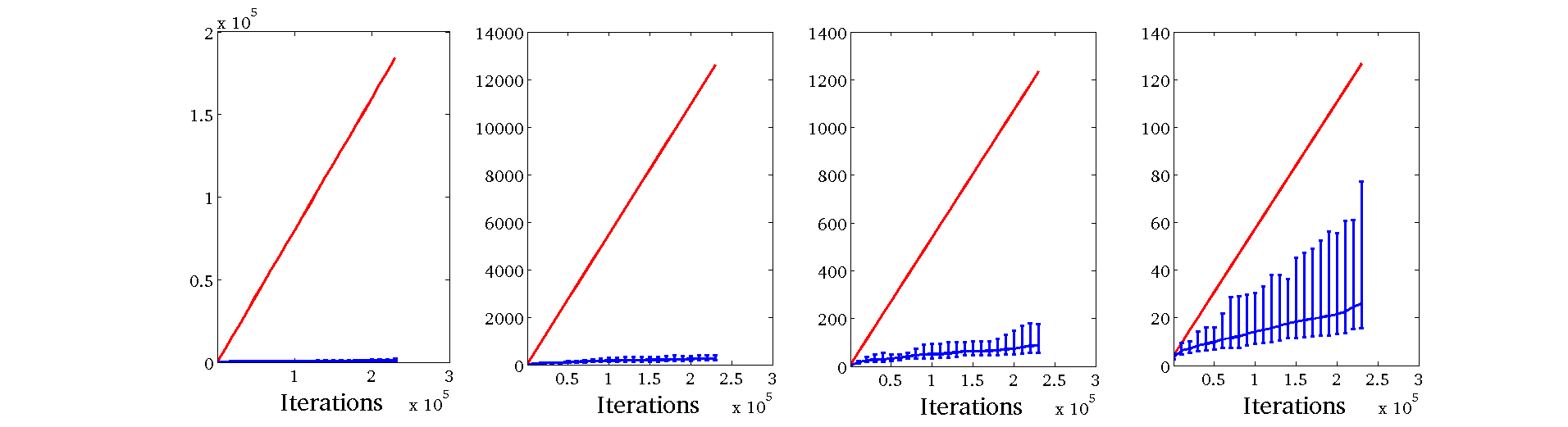}
\label{fig:twoNodesMarginalsSubsC}
           }

\end{center}
\caption{(a) Approximating a marginal distribution with Gibbs (blue) and herded Gibbs (red) for an MRF of two variables, constructed so as to make the move from state $(0,0)$ to $(1,1)$ progressively more difficult as $\epsilon$ decreases. The four columns, from left to right, are for $\epsilon = 0.1$, $\epsilon = 0.01$, $\epsilon = 0.001$ and $\epsilon = 0.0001$.
 Table~\ref{table:twoNodes} provides the joint distribution for these variables. The error bars for Gibbs correspond to one standard deviation.
 Rows (b) and (c) illustrate that the empirical convergence rate of herded Gibbs matches the expected theoretical rate. In the plots of rows (b) and (c), the upper-bound in the error of herded Gibbs was used to remove the oscillations so as to illustrate the behaviour of the algorithm more clearly.}
\label{fig:twoNodesMarginals}
\end{figure}

\subsection{MRF for Image Denoising}
\label{sec:denoising}

Next, we consider the standard setting of a grid-lattice MRF for image denoising.
Let us assume that we have a binary image corrupted by noise, and that we want to infer the original clean image. Let $X_i \in \{-1,+1\}$ denote the unknown true value of pixel $i$, and $y_i$ the observed, noise-corrupted value of this pixel.
We take advantage of the fact that neighboring pixels are likely to have the
same label by defining an MRF with an Ising prior. That is, we specify a rectangular 2D lattice
with the following pair-wise clique potentials:
\be
\psi_{ij}(x_i,x_j)
=
\begin{pmatrix}
e^{J_{ij}} & e^{-J_{ij}} \\
e^{-J_{ij}} & e^{J_{ij}}
\end{pmatrix}
\label{eqn:isingPot}
\ee
and joint distribution:
\be
p(\vx|\vJ) = \frac{1}{Z(\vJ)} \prod_{i \sim j} \psi_{ij}(x_i,x_j)
 = \frac{1}{Z(\vJ)} \exp\left(\frac{1}{2}\sum_{i \sim j} J_{ij} x_i x_j\right),
\label{eqn:isingPotJoint}
\ee
where $i \sim j$ is used to indicate that nodes $i$ and $j$ are connected.
The known parameters $J_{ij}$ establish the coupling strength between nodes $i$ and $j$.
Note that the matrix $\vJ$ is symmetric. If all the $J_{ij}>0$, then neighboring pixels are likely to be in the
same state.

The MRF model combines the Ising prior with a likelihood model as follows:
\bea
p(\vx,\vy) &=& p(\vx) p(\vy|\vx)
 = \left[ \frac{1}{Z}  \prod_{i \sim j} \psi_{ij}(x_i,x_j) \right].
\left[\prod_i p(y_i|x_i) \right]
\eea
The potentials $\psi_{ij}$ encourage label smoothness.
The likelihood terms
$p(y_i|x_i)$ are conditionally independent (e.g. Gaussians with known variance $\sigma^2$ and
mean $\vmu$ centered at each value of $x_i$, denoted $\mu_{x_i}$). In more precise terms,
\be
p(\vx,\vy|\vJ, \vmu, \sigma) 
 = \frac{1}{Z(\vJ,\vmu,\sigma)} \exp\left(\frac{1}{2}\sum_{i \sim j} J_{ij} x_i x_j - \frac{1}{2\sigma^2}\sum_{i}(y_i-\mu_{x_i})^2 \right).
\ee

When the coupling parameters $J_{ij}$ are identical, say $J_{ij}=J$, we have $\sum_{ij} J_{ij} f(x_i,x_j) = J \sum_{ij} f(x_i,x_j)$. Hence, different neighbor configurations result in the same value of $J \sum_{ij} f(x_i,x_j)$. If we store the conditionals for configurations with the same sum together, we only need to store as many conditionals as different possible values that the sum could take. This enables us to develop a shared version of herded Gibbs that is more memory efficient where we only maintain and update weights for \emph{distinct states} of the Markov blanket of each variable. 

In this exemplary image denoising experiment, noisy versions of the binary image, seen in Figure~\ref{fig:denoise} (left), were created through the addition of Gaussian noise, with varying $\sigma$. 
Figure~\ref{fig:denoise} (right) shows a corrupted image with $\sigma=4$. 
The $L_2$ 
reconstruction errors as a function of the number of iterations, for this example, are shown in
 Figure~\ref{fig:convergence}. The plot compares the herded Gibbs method against Gibbs and two versions of mean field with different damping factors~\cite{Murphy:2012}. The results demonstrate that the herded Gibbs techiques are among the best methods for solving this task.

\begin{figure}[t]
\begin{center}
  \includegraphics[width=0.45\textwidth]{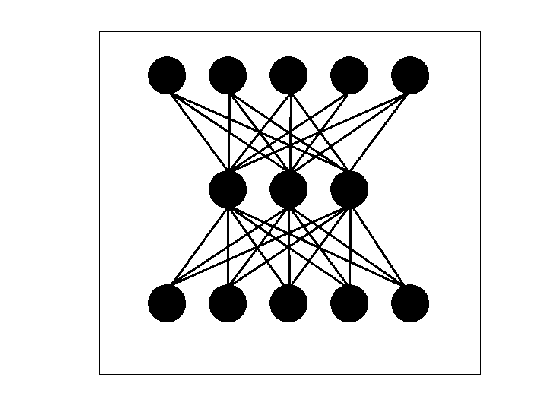}
  \includegraphics[width=0.45\textwidth]{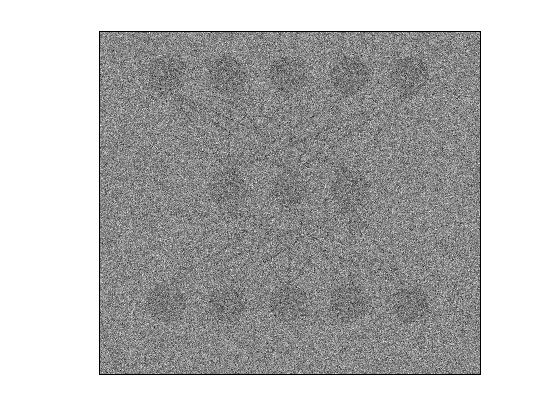}
\end{center}
\caption{Original image (left) and its corrupted version (right), with noise parameter
 $\sigma=4$.}
  \label{fig:denoise}
\end{figure}

\begin{figure}[t]
\begin{center}
\includegraphics[width=\textwidth]{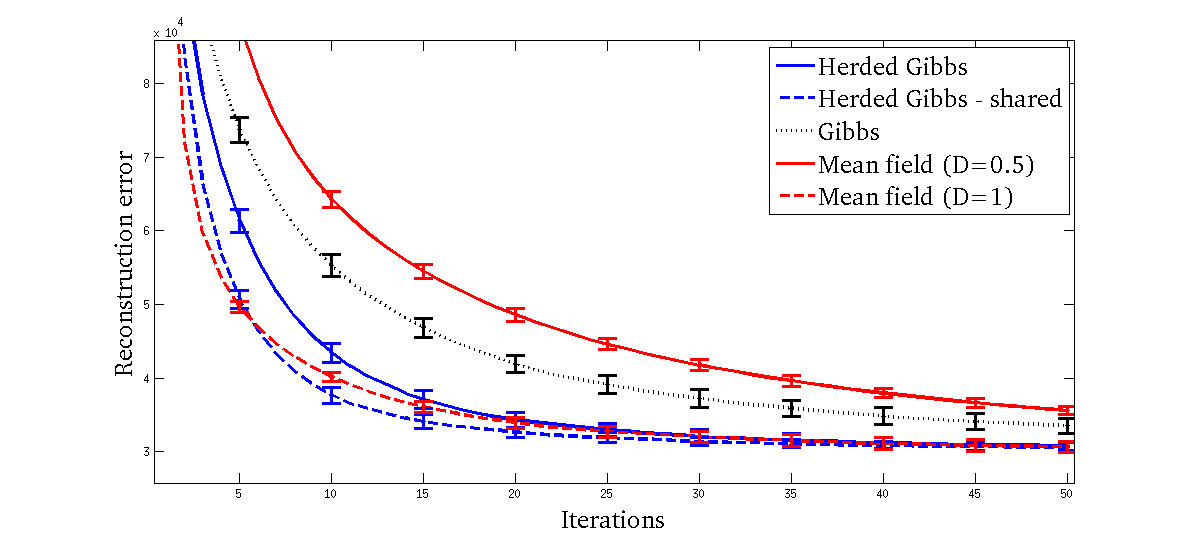}
\end{center}
\caption{Reconstruction errors for the image denoising task. The results are averaged across 10 corrupted images with Gaussian noise ${\cal N}(0,16)$. The error bars correspond to one standard deviation. Mean field requires the specification of the damping factor D. }
\label{fig:convergence}
\end{figure}

A comparison for different values $\sigma$ is presented in Table~\ref{table:mrf}. As expected mean field does well in the low-noise scenario, but the performance of the shared version of herded Gibbs as the noise increases is significantly better.

\begin{table}[t!]
\caption{Errors of image denoising example after 30 iterations (all measurements have been scaled by $\e{-3}$). We use an Ising prior with $J_{ij}=1$ and four Gaussian noise models with different $\sigma$'s. For each $\sigma$, we generated 10 corrupted images by adding Gaussian noise. The final results shown here are averages and standard deviations (in parentheses) across the 10 corrupted images. D denotes the damping factor in mean field.}
\label{table:mrf}
\begin{center}
{\small
\begin{tabular}{l|lllll}
\backslashbox{Method}{$\sigma$}	&2 &4 &6 &8
\\ \hline \\
 Herded Gibbs    &$21.58 (0.26)$	&$32.07 (0.98)$	&$47.52 (1.64)$	&$67.93 (2.78)$\\ 
 Herded Gibbs - shared    &$22.24 (0.29)$	&${\bf31.40} (0.59)$	&${\bf42.62} (1.98)$	&${\bf58.49} (2.86)$\\ 
 Gibbs   &$21.63 (0.28)$	&$37.20 (1.23)$	&$63.78 (2.41)$	&$90.27 (3.48)$\\ 
 Mean field (D=0.5)   &${\bf 15.52} (0.30)$	&$41.76 (0.71)$	&$76.24 (1.65)$	&$104.08 (1.93)$\\ 
 Mean field (D=1)   &$17.67 (0.40)$	&$32.04 (0.76)$	&$51.19 (1.44)$	&$74.74 (2.21)$\\ 
\end{tabular}
}
\end{center}
\end{table}

\subsection{CRF for Named Entity Recognition}
\label{sec:crf}

\begin{figure}[t]
\centering
\resizebox{0.95\textwidth}{!}{
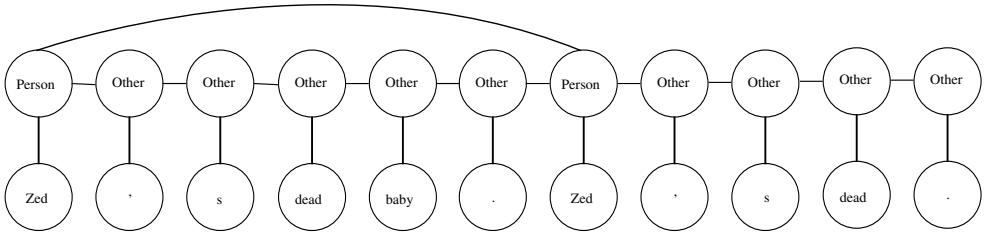
}
  \captionof{figure}{Typical skip-chain CRF model for named entity recognition.}
  \label{fig:crf}
  \end{figure}
  
Named Entity Recognition (NER) involves the identification of entities, such as people and locations, within a text sample. A conditional random fied (CRF) for NER models the relationship between entity labels and sentences with a conditional probability distribution:
$P(Y | X, \theta)$,
where $X$ is a sentence, $Y$ is a labeling, and $\theta$ is a vector of coupling parameters. The parameters, $\theta$, are feature weights and model relationships between variables $Y_{i}$ and $X_{j}$ or $Y_{i}$ and $Y_{j}$. A chain CRF only employs relationships between adjacent variables, whereas a skip-chain CRF can employ relationships between variables where subscripts $i$ and $j$ differ dramatically. Skip-chain CRFs are important in language tasks, such as NER and semantic role labeling, because they allow us to model long dependencies in a stream of words, see Figure~\ref{fig:crf}.

Once the parameters have been learned, the CRF can be used for inference; a labeling for some sentence $X$ is found by maximizing the above probability. 
Inference for CRF models in the NER domain is typically carried out with the Viterbi algorithm. However, if we want to accommodate long term dependencies, thus resulting in the so called skip-chain CRFs, Viterbi becomes prohibitively expensive. To surmount this problem, the Stanford named entity recognizer \cite{Finkel:2005} makes use of annealed Gibbs sampling.   

To demonstrate herded Gibbs on a practical application of great interest in text mining, we modify the standard inference procedure of the Stanford named entity recognizer by replacing the annealed Gibbs sampler with the herded Gibbs sampler. The herded Gibbs sampler in not annealed. To find the maximum a posteriori sequence $Y$, we simply choose the sample with highest joint discrete probability.
In order to be able to compare against Viterbi, we have purposely chosen to use single-chain CRFs. We remind the reader, however, that the herded Gibbs algorithm could be used in cases where Viterbi inference is not possible.

We used the pre-trained 3-class CRF model in the Stanford NER package \cite{Finkel:2005}. 
This model is a linear chain CRF with pre-defined features and pre-trained feature weights, $\theta$. 
For the test set, we used the corpus for the NIST 1999 IE-ER Evaluation. Performance is measured in per-entity 
$F_1$ $\left( F_1 = 2 \cdot \frac{\text{precision} \cdot \text{recall}}{\text{precision} + \text{recall} } \right)$. 
For all the methods, except Viterbi, we show $F_1$ scores after 100, 400 and 800 iterations
in Table~\ref{table:crf}. For Gibbs, the results shown are the averages and standard deviations over 5 random runs. We used a linear annealing schedule for Gibbs.
As the results illustrate, herded Gibbs attains the same accuracy as Viterbi and it is faster than annealed Gibbs. Unlike Viterbi, herded Gibbs can be easily applied to skip-chain CRFs.  After only 400 iterations (90.5 seconds), herded Gibbs already achieves an $F_{1}$ score of 84.75, while Gibbs, even after 800 iterations (115.9 seconds) only achieves an $F_{1}$ score of 84.61. The experiment thus clearly demonstrates that (i) herded Gibbs does no worse than the optimal solution, Viterbi, and (ii) herded Gibbs yields more accurate results for the same amount of computation.
Figure~\ref{fig:NERExample} provides a representative NER example of the performance of Gibbs, herded Gibbs and Viterbi (all methods produced the same annotation for this short example).

\begin{table}[H]
\caption{Gibbs, herded Gibbs and Viterbi for the NER task. 
The average computational time each approach took to do inference for the entire test set is listed (in square brackets). After only 400 iterations (90.48 seconds), herded Gibbs already achieves an $F_{1}$ score of 84.75, while Gibbs, even after 800 iterations (115.92 seconds) only achieves an $F_{1}$ score of 84.61. For the same computation, herded Gibbs is more accurate than Gibbs. }
\label{table:crf}
\begin{center}
{\small
\makebox[\textwidth][c]{
\begin{tabular}{l|lll}
\backslashbox{Method}{Iterations}	&100 &400 &800
\\ \hline \\
Annealed Gibbs		&$84.36 (0.16)$	[$55.73$s]	&$84.51 (0.10)$	[$83.49$s]	&$84.61 (0.05)$	[$115.92$s]\\
Herded Gibbs		&$84.70$	[$59.08$s]	&$84.75$	[$90.48$s]	&$84.81$ 	[$132.00$s]\\
Viterbi		&&&$84.81$[$46.74$s]  \\
\end{tabular}
}}
\end{center}
\end{table}

\begin{figure}[t]
\begin{quote}

{\footnotesize \textcolor{blue}{
"Pumpkin" (\textit{\textcolor{red}{Tim Roth}}) and "Honey Bunny" (\textit{\textcolor{red}{Amanda Plummer}}) are having breakfast in a diner. They decide to rob it after realizing they could make money off the customers as well as the business, as they did during their previous heist. Moments after they initiate the hold-up, the scene breaks off and the title credits roll.
As \textit{\textcolor{red}{Jules Winnfield}} (\textit{\textcolor{red}{Samuel L. Jackson}}) drives, \textit{\textcolor{red}{Vincent Vega}} (\textit{\textcolor{red}{John Travolta}}) talks about his experiences in \textbf{\textcolor{green}{Europe}}, from where he has just returned: the hash bars in \textbf{\textcolor{green}{Amsterdam}}, the French \underline{\textcolor{orange}{McDonald's}} and its "Royale with Cheese".
}}
\end{quote}
\caption{Results for the application of the NER CRF to a random Wikipedia sample~\cite{Wiki:2012:Online}. Entities are automatically classified as \textit{\textcolor{red}{Person}}, \textbf{\textcolor{green}{Location}} and \underline{\textcolor{orange}{Organization}}.  }
\label{fig:NERExample}
\end{figure}

\section{Conclusions and Future Work}
\label{sec:conclusions}

In this paper, we introduced herded Gibbs, a deterministic variant of the popular Gibbs sampling algorithm. While Gibbs relies on drawing samples from the {full-conditionals} at random, herded Gibbs generates the samples by matching the full-conditionals. Importantly, the herded Gibbs algorithm is very close to the Gibbs algorithm and hence retains its simplicity of implementation. 

The synthetic, denoising and named entity recognition experiments provided evidence that herded Gibbs outperforms Gibbs sampling. However, as discussed, herded Gibbs requires storage of the conditional distributions for all instantiations of the neighbors in the worst case. This storage requirement indicates that it is more suitable for sparse probabilistic graphical models, such as the CRFs used in information extraction. 
At the other extreme, the paper advanced the theory of deterministic sampling by showing that herded Gibbs converges with rate $O(1/T)$ for models with independent variables and fully-connected models. Thus, there is gap between theory and practice that needs to be narrowed. We do not anticipate that this will be an easy task, but it is certainly a key direction for future work. 

We should mention that it is also possible to design parallel versions of herded Gibbs in a Jacobi fashion.
We have indeed studied this and found that these are  less efficient than the Gauss-Seidel version of herded Gibbs discussed in this paper. However, if many cores are available, we strongly recommend the Jacobi (asynchronous) implementation as it will likely outperform the Gauss-Seidel (synchronous) implementation.

The design of efficient herding algorithms for densely connected probabilistic graphical models remains an important area for future research.  Such algorithms, in conjunction with Rao Blackwellization, would enable us to attack many statistical inference tasks, including Bayesian variable selection and Dirichlet processes. 

There are also interesting connections with other algorithms to explore. If, for a fully connected graphical model, we build a new graph where every state is a node and directed connections exist between nodes that can be reached with a single herded Gibbs update, then herded Gibbs becomes equivalent to the Rotor-Router model of Alex Holroyd and Jim Propp\footnote{We thank Art Owen for pointing out this connection.} \cite{holroyd2010rotor}. This deterministic analogue of a random walk has provably superior concentration rates for quantities such as normalized hitting frequencies, hitting times and occupation frequencies. In line with our own convergence results, it is shown that discrepancies in these quantities decrease as $O(1/T)$ instead of the usual $O(1/\sqrt{T})$. We expect that many of the results from this literature apply to herded Gibbs as well. The connection with the work of Art Owen and colleagues, see for example \cite{Chen:2011}, also needs to be explored further. Their work uses \emph{completely uniformly distributed (CUD) sequences} to drive Markov chain Monte Carlo schemes. It is not clear, following discussions with Art Owen, that CUD sequences can be constructed in a greedy way as in herding.

%
%
%


\bibliography{herd,Refs_yutian}
\bibliographystyle{plain}

\appendix

\section{Proof of Theorem \ref{thm:convergence_empty_graph}}
We first show that the weight dynamics of a one-variable herding algorithm are restricted to an invariant interval of length $1$.

\begin{lemma} \label{lem:single_node} If $w$ is the weight of the herding dynamics of a single binary variable $X$ with probability $P(X=1)=\pi$, and $w^{(s)} \in (\pi-1, \pi]$ at some step $s\geq 0$, then $w^{(t)}\in(\pi-1,\pi],\forall t\geq s$. Moreover, for $T\in\mathbb{N}$, we have:
\begin{align}
\sum_{t=s+1}^{s+T} \mathbb{I}[X^{(t)}=1] &\in [T\pi - 1, T\pi + 1] \label{eqn:lemma_single_node_1} \\
\sum_{t=s+1}^{s+T} \mathbb{I}[X^{(t)}=0] &\in [T(1-\pi) - 1, T(1-\pi) + 1] .\label{eqn:lemma_single_node_2}
\end{align}
\end{lemma}
\begin{proof}
We first show that $w\in (\pi - 1, \pi], \forall t \geq s$. This is easy to observe by induction as $w^{(s)} \in (\pi-1, \pi]$ and if $w^{(t)} \in (\pi-1, \pi]$ for some $t\geq s$, then, following Equation \ref{eq:singleNodeDynamicsA}, we have:
\begin{equation}
w^{(t+1)} = \left\{\begin{array}{ll}
w^{(t)} + \pi - 1 \in (\pi - 1, 2\pi - 1] \subseteq (\pi-1, \pi] & \mbox {if } w^{(t)} > 0 \\
w^{(t)} + \pi \in (2\pi-1, \pi] \subseteq (\pi-1, \pi] & \mbox {otherwise}.
\end{array}\right.
\end{equation}
Summing up both sides of Equation \ref{eq:singleNodeDynamicsA} over $t$ immediately gives us the result of Equation \ref{eqn:lemma_single_node_1} since:
\begin{equation}
T \pi - \sum_{t=s+1}^{s+T} \mathbb{I}[X^{(t)}=1] = w^{(s+T)}-w^{(s)} \in [-1, 1].
\end{equation}
In addition, Equation \ref{eqn:lemma_single_node_2} follows by observing that $\mathbb{I}[X^{(t)}=0] = 1 - \mathbb{I}[X^{(t)}=1]$.
\end{proof}

When $w$ is outside the invariant interval, it is easy to observe that $w$ will move into it monotonically at a linear speed in a transient period. So we will always consider an initialization of $w\in(\pi-1,\pi]$ from now on.

Equivalently, we can take a one-to-one mapping $w \leftarrow w \mod 1$ (we define $1 \mod 1 = 1$) and think of $w$ as updated by a constant translation vector in a circular unit interval $(0,1]$ as shown in Figure \ref{fig:univariate_2}. That is, 
\bea
w^{(t)} =  (w^{(t-1)} + \pi) \mod 1, \quad x^{(t)} = \left\{\begin{array}{rl}
1 & \mbox {if $w^{(t-1)} < \pi$} \\
0 & \mbox {otherwise}
\end{array}\right.
\label{eq:singleNodeDynamicsB}
\eea

\begin{figure}[h!]
 \begin{minipage}[t]{.5\textwidth}
 \centering
 \includegraphics[width=.9\textwidth] {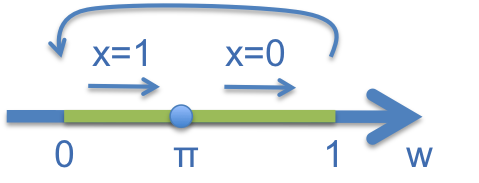}
 \caption{Equivalent weight dynamics for\newline a single variable.}
 \label{fig:univariate_2}
 \end{minipage}
 \begin{minipage}[t]{.5\textwidth}
 \centering
 \includegraphics[width=.9\textwidth] {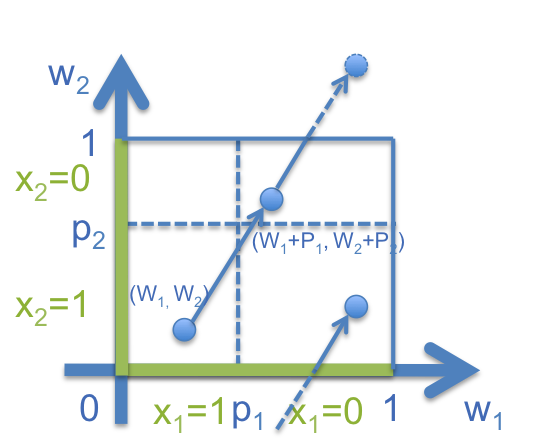}
 \caption{Dynamics of herding with two independent variables.}
 \label{fig:bivariable}
 \end{minipage}
\end{figure}


We are now ready to give the proof of Theorem \ref{thm:convergence_empty_graph}.
\begin{proof}[Proof of Theorem \ref{thm:convergence_empty_graph}]
For an empty graph of $N$ independent vertices, the dynamics of the weight vector $\mathbf{w}$ are equivalent to a constant translation mapping in an $N$-dimensional circular unit space $(0,1]$, as shown in Figure \ref{fig:bivariable}:
\bea
\mathbf{w}^{(t)} &= & (\mathbf{w}^{(t-1)} + \boldsymbol{\pi}) \mod 1 \nonumber \\
&=& (\mathbf{w}^{(0)} + t \boldsymbol{\pi}) \mod 1, \quad x_i^{(t)} = \left\{\begin{array}{rl}
1 & \mbox {if $w_i^{(t-1)} < \pi_i$} \\
0 & \mbox {otherwise}
\end{array}\right., \forall 1\leq i\leq N
\label{eq:singleNodeDynamicsB}
\eea
The Kronecker-Weyl theorem \cite{weyl} states that the sequence $\tilde{\mathbf{w}}^{(t)}=t\mathbf{\pi} \mod 1, t\in\mathbb{Z}^+$ is equidistributed (or uniformly distributed) on $(0, 1]$ if and only if $(1, \pi_1, \dots, \pi_N)$ is rationally independent. Since we can define a one-to-one volume preserving transformation between $\tilde{\mathbf{w}}^{(t)}$ and $\mathbf{w}^{(t)}$ as $(\tilde{\mathbf{w}}^{(t)} + \mathbf{w}^{(0)}) \mod 1 = \mathbf{w}^{(t)}$, the sequence of weights $\{\mathbf{w}^{(t)}\}$ is also uniformly distributed in $(0, 1]^N$.

Define the mapping from a state value $x_i$ to an interval of $w_i$ as 
\begin{equation}
A_i(x)=\left\{\begin{array}{rl}
(0, \pi_i] & \mbox {if $x = 1$} \\
(\pi_i, 1] & \mbox {if $x = 0$}
\end{array}\right.
\end{equation}
and let $|A_i|$ be its measure. We obtain the limiting distribution of the joint state as
\bea
\lim_{T\rightarrow\infty}P_T^{(\tau)}(\mathbf{X}=\mathbf{x})&=&\lim_{T\rightarrow\infty}\frac{1}{T}\sum_{t=1}^T \mathbb{I}\left[\mathbf{w}^{(t-1)}\in \prod_{i=1}^N A_i(x_i)\right] \nonumber \\
&= & \prod_{i=1}^N |A_i(x_i)| \nonumber \\
&=& \prod_{i=1}^N \pi(X_i=x_i)\nonumber  \\
&=& \pi(\mathbf{X}=\mathbf{x})
\eea
\end{proof}

\section{Proof of Theorem \ref{thm:convergenceHerdedGibbs}}

In this appendix, we give an upper bound for the convergence rate of the sampling distribution in fully connected graphs. As herded Gibbs sampling is deterministic, the distribution of a variable's state at every iteration  degenerates to a single state. As such, we study here the empirical distribution of a collection of samples.

The structure of the proof is as follows (with notation defined in the next subsection): We study the distribution distance between the invariant distribution $\pi$ and the empirical distribution of $T$ samples collected starting from sweep $\tau$, $P^{(\tau)}_T$. We show that the distance decreases as $\tau \Rightarrow \tau+1$ with the help of an auxiliary regular Gibbs sampling Markov chain initialized at $\pi^{(0)}=P^{(\tau)}_T$, as shown in Figure \ref{fig:distr_dist}. On the one hand, the distance between the regular Gibbs chain after one iteration, $\pi^{(1)}$, and $\pi$ decreases according to the geometric convergence property of MCMC algorithms on compact state spaces. On the other hand, we show that in one step the distance between $P^{(\tau+1)}_T$ and $\pi^{(1)}$ increases by at most $O(1/T)$. Since the $O(1/T)$ distance term dominates the exponentially small distance term,  the distance between $P^{(\tau+1)}_T$ and $\pi$ is bounded by $O(1/T)$. Moreover, after a short burn-in period, $L=O(\log(T))$, the empirical distribution $P^{(\tau+L)}_T$ will have an approximation error in the order of $O(1/T)$.

\begin{figure}[htp!]
\centering
\includegraphics[width=.6\linewidth]{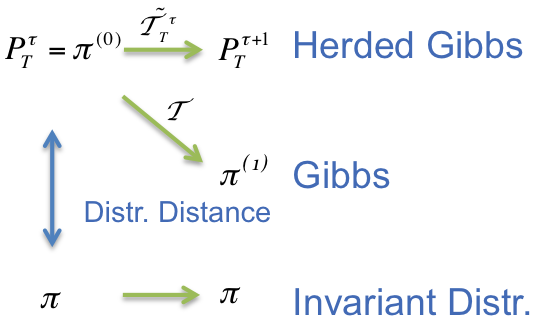} 
\caption{Transition kernels and relevant distances for the proof of Theorem~\ref{thm:convergenceHerdedGibbs}.}
\label{fig:distr_dist}
\end{figure}

\subsection{Notation}
Assume without loss of generality that in the systematic scanning policy, the variables are sampled in the order $1,2,\cdots,N$.

\subsubsection{State Distribution}

\begin{itemize}
\item Denote by $\mathcal{X}_+$ the support of the distribution $\pi$, that is, the set of states with positive probability.
\item We use $\tau$ to denote the time in terms of sweeps over all of the $N$ variables, and $t$ to denote the time in terms of steps where one step constitutes the updating of one variable. For example, at the end of $\tau$ sweeps, we have $t=\tau N$.
\item Recall the sample/empirical distribution, $P^{(\tau)}_T$,  presented in Definition~\ref{def:hGDist}.  Figure~\ref{fig:distr}  provides a visual interpretation of the definition.   
\begin{figure}[htp!]
\centering
\includegraphics[width=.8\linewidth]{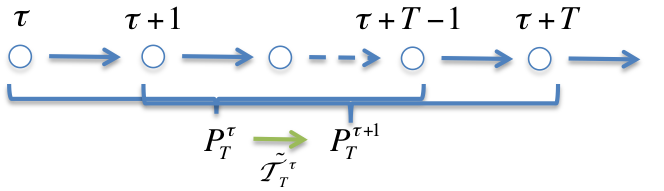} 
\caption{Distribution over time at the end of every sweep.}
\label{fig:distr}
\end{figure}
\item Denote the sample/empirical distribution at the $i^{th}$ step within a sweep as $P^{(\tau)}_{T,i}, \tau\geq 0, T>0, 0\leq i \leq N$, as shown in Figure \ref{fig:distr_in_sweep}:
$$
P^{(\tau)}_{T,i}(\mathbf{X}=\mathbf{x})=\frac{1}{T}\sum_{k=\tau}^{\tau+T-1}\mathbb{I}(\mathbf{X}^{(kN+i)}=\mathbf{x}).
$$
This is the distribution of $T$ samples collected at the $i^{th}$ step of every sweep, starting from the $\tau^{th}$ sweep. Clearly, $P_{T}^{(\tau)} = P_{T,0}^{(\tau)} = P_{T,N}^{(\tau-1)}$.
\begin{figure}[htp!]
\centering
\includegraphics[width=.8\linewidth]{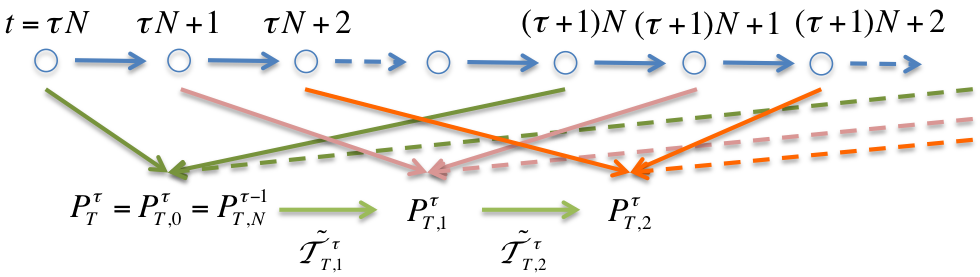} 
\caption{Distribution over time within a sweep.}
\label{fig:distr_in_sweep}
\end{figure}

\item Denote the distribution of a regular Gibbs sampling Markov chain after $L$ sweeps of updates over the $N$ variables with $\pi^{(L)}, L\geq 0$. 

For a given time $\tau$, we construct a Gibbs Markov chain with initial distribution $\pi^{0}=P_{T}^{(\tau)}$ and the same scanning order of herded Gibbs, as shown in Figure \ref{fig:distr_dist}.
\end{itemize}

\subsubsection{Transition Kernel}
\begin{itemize}
\item Denote the transition kernel of regular Gibbs for the step of updating a variable $X_i$ with $\mathcal{T}_i$, and for a whole sweep with $\mathcal{T}$. 

By definition, $\pi^{0}\mathcal{T} = \pi^{1}$. The transition kernel for a single step can be represented as a $2^N\times 2^N$ matrix:
\begin{equation}
\mathcal{T}_i(\mathbf{x}, \mathbf{y}) = \left\{\begin{array}{rl}
0 & \mbox {if } \mathbf{x}_{-i} \neq \mathbf{y}_{-i} \\
\pi(X_i = y_i | \mathbf{x}_{-i}) & \mbox {otherwise}
\end{array}\right.,1\leq i\leq N, \mathbf{x}, \mathbf{y} \in \{0, 1\}^N
\end{equation}
where $\mathbf{x}$ is the current state vector of $N$ variables, $\mathbf{y}$ is the state of the next step, and $\vx_{-i}$ denotes all the components of $\vx$ excluding the $i^{th}$ component.  If $\pi(\mathbf{x}_{-i})=0$, the conditional probability is undefined and we set it with an arbitrary distribution. Consequently, $\mathcal{T}$ can also be represented as:
$$
\mathcal{T} = \mathcal{T}_1 \mathcal{T}_2 \cdots \mathcal{T}_N.
$$

\item Denote the Dobrushin ergodic coefficient \cite{bremaud1999markov} of the regular Gibbs kernel with $\eta \in [0, 1]$. When $\eta < 1$, the regular Gibbs sampler has a geometric rate of convergence of
\begin{equation}
d_v(\pi^{(1)} - \pi) = d_v(\mathcal{T}\pi^{(0)} - \pi) \leq \eta d_v(\pi^{(0)} - \pi), \forall \pi^{(0)} \label{eqn:dvgibbs_inv}.
\end{equation}
A common sufficient condition for $\eta < 1$ is that $\pi(\mathbf{X})$ is strictly positive.

\item Consider the sequence of sample distributions $P^{(\tau)}_T, \tau=0,1,\cdots$ in Figures \ref{fig:distr} and \ref{fig:distr_in_sweep}. We define the transition kernel of herded Gibbs for the step of updating variable $X_i$ with $\tilde{\mathcal{T}}^{(\tau)}_{T, i}$, and for a whole sweep with $\tilde{\mathcal{T}}^{(\tau)}_{T}$. 

Unlike regular Gibbs, the transition kernel is not homogeneous. It depends on both the time $\tau$ and the sample size $T$. Nevertheless, we can still represent the single step transition kernel as a matrix:
\begin{equation}
\tilde{\mathcal{T}}^{(\tau)}_{T, i}(\mathbf{x}, \mathbf{y}) = \left\{\begin{array}{rl}
0 & \mbox {if } \mathbf{x}_{-i} \neq \mathbf{y}_{-i} \\
P^{(\tau)}_{T, i}(X_i = y_i | \mathbf{x}_{-i}) & \mbox {if } \mathbf{x}_{-i} = \mathbf{y}_{-i}
\end{array}\right.,1\leq i\leq N,\mathbf{x}, \mathbf{y} \in \{0, 1\}^N, 
\end{equation}
where $P^{(\tau)}_{T, i}(X_i = y_i | \mathbf{x}_{-i})$ is defined as:
\begin{align}
& P^{(\tau)}_{T, i}(X_i = y_i | \mathbf{x}_{-i}) = \frac{N_{\textrm{num}}}{N_{\textrm{den}}} \nonumber\\ 
& N_{\textrm{num}} = T P^{(\tau)}_{T, i}(\mathbf{X}_{-i}=\mathbf{x}_{-i}, X_i = y_i) = \sum_{k=\tau}^{\tau+T-1} \mathbb{I}(\mathbf{X}_{-i}^{(kN + i)}=\mathbf{x}_{-i}, X_i^{(kN + i)}=y_i) \nonumber\\
& N_{\textrm{den}} = T P^{(\tau)}_{T, i-1}(\mathbf{X}_{-i}=\mathbf{x}_{-i}) = \sum_{k=\tau}^{\tau+T-1} \mathbb{I}(\mathbf{X}_{-i}^{(kN + i -1)}=\mathbf{x}_{-i}), 
\end{align}
where $N_{\textrm{num}}$ is the number of occurrences of a joint state, and $N_{\textrm{den}}$ is the number of occurrences of a conditioning state in the previous step. When $\pi(\mathbf{x}_{-i})=0$, we know that $N_{\textrm{den}} = 0$ with a proper initialization of herded Gibbs, and we simply set $\tilde{\mathcal{T}}^{(\tau)}_{T, i}=\mathcal{T}_{i}$ for these entries. It is not hard to verify the following identity by expanding every term with its definition $$
P_{T,i}^{(\tau)} = P_{T,i-1}^{(\tau)} \tilde{\mathcal{T}}^{(\tau)}_{T, i}
$$ and consequently, 
$$
P_{T}^{(\tau+1)} = P_{T}^{(\tau)} \tilde{\mathcal{T}}^{(\tau)}_{T}
$$
with 
$$
\tilde{\mathcal{T}}^{(\tau)}_{T} = \tilde{\mathcal{T}}^{(\tau)}_{T, 1} \tilde{\mathcal{T}}^{(\tau)}_{T, 2} \cdots \tilde{\mathcal{T}}^{(\tau)}_{T, N}.
$$
\end{itemize}

\subsection{Linear Visiting Rate}
We prove in this section that every joint state in the support of the target distribution is visited, at least, at a linear rate. This result will be used to measure the distance between the Gibbs and herded Gibbs transition kernels.


\begin{proposition} \label{prop:reachability}
If a graph is fully connected, herded Gibbs sampling scans variables in a fixed order, and the corresponding Gibbs sampling Markov chain is irreducible, then for any state $\mathbf{x}\in\mathcal{X}_+$ and any index $i\in[1,N]$, the state is visited at least at a linear rate. Specifically,
\begin{align}
& \exists l > 0, B > 0, s.t., \forall i \in [1, N], \mathbf{x}\in\mathcal{X}_+, T \in \mathbb{N}, s \in \mathbb{N}\nonumber \\
& \sum_{k=s}^{s+T-1} \mathbb{I}\left[\mathbf{X}^{(t=Nk+i)}=\mathbf{x}\right] \geq l T - B
\end{align}
\end{proposition}

Denote the minimum nonzero conditional probability as $$
\pi_{\textrm{min}}=\min_{1\leq i \leq N,\pi(x_i|\mathbf{x}_{-i})>0}\pi(x_i|\mathbf{x}_{-i}).
$$

The following lemma, which is needed to prove Proposition~\ref{prop:reachability}, gives an inequality between the number of visits of two sets of states in consecutive steps.
\begin{lemma} \label{lem:iterative_step}
For any integer $i\in [1,N]$ and two sets of states $\mathbb{X}, \mathbb{Y}\subseteq \mathcal{X}_+$ with a mapping $F:\mathbb{X}\rightarrow\mathbb{Y}$ that satisfies the following condition:
\begin{equation}
\forall \mathbf{x}\in\mathbb{X}, \mathbf{F(\mathbf{x})}_{-i}=\mathbf{x}_{-i},\quad \cup_{\mathbf{x}\in\mathbb{X}}F(\mathbf{x})=\mathbb{Y} \label{eqn:lemma_iterative_step_condion},
\end{equation}
we have that, for any $s\geq0$ and $ T>0$, the number of times $\mathbb{Y}$ is visited in the set of steps $C_i=\{t=kN+i:s\leq k\leq k+T-1\}$ is lower bounded by a function of the number of times $\mathbb{X}$ is visited in the previous steps $C_{i-1}=\{t=kN+i-1:s\leq k\leq k+T-1\}$ as:
\begin{equation}
\sum_{t\in C_i}\mathbb{I}\left[\mathbf{X}^{(t)}\in\mathbb{Y}\right] \geq \pi_{\textrm{min}} \sum_{t\in C_{i-1}}\mathbb{I}\left[\mathbf{X}^{(t)}\in\mathbb{X}\right] - |\mathbb{Y}|
\end{equation}
\end{lemma}
\begin{proof}
As a complement to Condition~\ref{eqn:lemma_iterative_step_condion}, we can define $F^{-1}$ as the inverse mapping from $\mathbb{Y}$ to subsets of $\mathbb{X}$ so that for any $\mathbf{y}\in\mathbb{Y}$, $\mathbf{x}\in F^{-1}(\mathbf{y})$, we have $\mathbf{x}_{-i}=\mathbf{y}_{-i}$, and $\cup_{\mathbf{y}\in\mathbb{Y}}F^{-1}(\mathbf{y})=\mathbb{X}$.

Consider any state $\mathbf{y}\in\mathbb{Y}$, when $\mathbf{y}$ is visited in $C_i$, the weight $w_{i,\mathbf{y}_{-i}}$ is active. Let us denote the set of all the steps in $[sN+1, s(N+T)+N]$ when $w_{i,\mathbf{y}_{-i}}$ is active by $C_i(\mathbf{y}_{-i})$, that is, $C_i(\mathbf{y}_{-i})=\{t:t\in C_i, \mathbf{X}^{(t)}_{-i}=\mathbf{y}_{-i}\}$. Applying Lemma \ref{lem:single_node} we get
\begin{equation}
\sum_{t\in C_i}\mathbb{I}\left[\mathbf{X}^{(t)}=\mathbf{y}\right] \geq \pi(y_i|\mathbf{y}_{-i}) |C_i(\mathbf{y}_{-i})|-1 \geq \pi_{\textrm{min}} |C_i(\mathbf{y}_{-i})|-1. \label{eqn:lemma_iterative_step_single_state}
\end{equation}
Since the variables $\mathbf{X}_{-i}$ are not changed at steps in $C_i$, we have
\begin{equation}
|C_i(\mathbf{y}_{-i})|=\sum_{t\in C_{i-1}}\mathbb{I}\left[\mathbf{X}^{(t)}_{-i}=\mathbf{y}_{-i}\right] \geq \sum_{t\in C_{i-1}}\mathbb{I}\left[\mathbf{X}^{(t)}\in F^{-1}(\mathbf{y})\right].
\end{equation}
Combining the fact that $\cup_{\mathbf{y}\in\mathbb{Y}}F^{-1}(\mathbf{y})=\mathbb{X}$ and summing up both sides of Equation~\ref{eqn:lemma_iterative_step_single_state} over $\mathbb{Y}$ proves the lemma:
\begin{equation}
\sum_{t\in C_i}\mathbb{I}\left[\mathbf{X}^{(t)}\in\mathbb{Y}\right] \geq \sum_{\mathbf{y} \in\mathbb{Y}} \left(\pi_{\textrm{min}} \sum_{t\in C_{i-1}}\mathbb{I}\left[\mathbf{X}^{(t)}\in F^{-1}(\mathbf{y})\right] - 1 \right) \geq \pi_{\textrm{min}} \sum_{t\in C_{i-1}}\mathbb{I}\left[\mathbf{X}^{(t)}\in\mathbb{X}\right] - |\mathbb{Y}|.
\end{equation}
\end{proof}

\begin{remark}
A fully connected graph is a necessary condition for the application of Lemma \ref{lem:single_node} in the proof. If a graph is not fully connected ($N(i)\neq -i$), a weight $w_{i, \mathbf{y}_{N(i)}}$ may be shared by multiple full conditioning states. In this case $C_i(\mathbf{y}_{-i})$ is no longer a consecutive sequence of times when the weight is updated, and Lemma \ref{lem:single_node} does not apply here.
\end{remark}

Now let us prove Proposition \ref{prop:reachability} by iteratively applying Lemma \ref{lem:iterative_step}.

\begin{proof}[Proof of Proposition \ref{prop:reachability}]
Because the corresponding Gibbs sampler is irreducible and any Gibbs sampler is aperiodic, there exists a constant $t^* > 0$ such that for any state $\mathbf{y}\in\mathcal{X}_+$, and any step in a sweep, $i$, we can find a path of length $t^*$ for any state $\mathbf{x}\in\mathcal{X}_+$ with a positive transition probability, $Path(\mathbf{x})=(\mathbf{x}=\mathbf{x}(0), \mathbf{x}(1), \dots, \mathbf{x}(t^*)=\mathbf{y})$, to connect from $\mathbf{x}$ to $\mathbf{y}$, where each step of the path follows the Gibbs updating scheme. For a strictly positive distribution, the minimum value of $t^*$ is $N$.

Denote $\tau^*=\lceil t^*/N \rceil$ and the $j^{th}$ element of the path $Path(\mathbf{x})$ as $Path(\mathbf{x}, j)$. We can define $t^*+1$ subsets $S_j\subseteq \mathcal{X}_+, 0\leq j\leq t^*$ as the union of all the $j^{th}$ states in the path from any state in $\mathcal{X}_+$:
$$
S_j = \cup_{\mathbf{x}\in\mathcal{X}_+}Path(\mathbf{x}, j)
$$
By definition of these paths, we know $S_0=\mathcal{X}_+$ and $S_{t^*}=\{\mathbf{y}\}$, and there exits an integer $i(j)$ and a mapping $F_j:S_{j-1}\rightarrow S_j,\forall j$ that satisfy the condition in Lemma \ref{lem:iterative_step} ($i(j)$ is the index of the variable to be updated, and the mapping is defined by the transition path). Also notice that any state in $S_j$ can be different from $\mathbf{y}$ by at most $\min\{N,t^*-j\}$ variables, and therefore $|S_j|\leq 2^{\min\{N,t^*-j\}}$.

Let us apply Lemma \ref{lem:iterative_step} recursively from $j=t^*$ to $1$ as
\begin{align}
\sum_{k=s}^{s+T-1}\mathbb{I}\left[\mathbf{X}^{(t=Nk+i)}=\mathbf{y}\right] &\geq
\sum_{k=s+\tau^*}^{s+T-1}\mathbb{I}\left[\mathbf{X}^{(t=Nk+i)}=\mathbf{y}\right] \nonumber\\
&= \sum_{k=s+\tau^*}^{s+T-1}\mathbb{I}\left[\mathbf{X}^{(t=Nk+i)}\in S_{t^*}\right] \nonumber\\
&\geq \pi_{\textrm{min}} \sum_{k=s+\tau^*}^{s+T-1}\mathbb{I}\left[\mathbf{X}^{(t=Nk+i-1)}\in S_{t^*-1}\right]-|S_{t^*}| \nonumber\\
&\geq \cdots \nonumber\\
&\geq \pi_{\textrm{min}}^{t^*} \sum_{k=s+\tau^*}^{s+T-1}\mathbb{I}\left[\mathbf{X}^{(t=Nk+i-t^*)}\in S_0=\mathcal{X}_+\right]-\sum_{j=0}^{t^*-1}\pi_{\textrm{min}}^{j}|S_{t^*-j}| \nonumber\\
&\geq \pi_{\textrm{min}}^{t^*}(T-\tau^*)-\sum_{j=0}^{t^*-1}\pi_{\textrm{min}}^{j} 2^{\min\{N,j\}}.
\end{align}
The proof is concluded by choosing the constants
\begin{equation}
l = \pi_{\textrm{min}}^{t^*},\quad B=\tau^*\pi_{\textrm{min}}^{t^*}+\sum_{j=0}^{t^*-1}\pi_{\textrm{min}}^{j} 2^{\min\{N,j\}}. \label{eqn:lB}
\end{equation}

\end{proof}

\subsection{Herded Gibbs's Transition Kernel $\tilde{\mathcal{T}}^{(\tau)}_T$ is an Approximation to $\mathcal{T}$}

The following proposition shows that $\tilde{\mathcal{T}}^{(\tau)}_T$ is an approximation to the regular Gibbs sampler's transition kernel $\mathcal{T}$ with an error of $O(1/T)$.

\begin{proposition} \label{prop:err_onestep}
For a fully connected graph, if the herded Gibbs has a fixed scanning order and the corresponding Gibbs sampling Markov chain is irreducible, then for any $\tau\geq 0$, $T\geq T^*\defeq \frac{2B}{l}$ where $l$ and $B$ are the constants in Proposition \ref{prop:reachability}, the following inequality holds:
\begin{equation}
\|\tilde{\mathcal{T}}^{(\tau)}_{T} - \mathcal{T}\|_{\infty} \leq \frac{4N}{lT}
\end{equation}
\end{proposition}
\begin{proof}

When $\mathbf{x}\not\in\mathcal{X}_+$, we have the equality $\tilde{\mathcal{T}}^{(\tau)}_{T,i}(\mathbf{x},\mathbf{y})=\mathcal{T}_i(\mathbf{x},\mathbf{y})$ by definition. When $\mathbf{x}\in\mathcal{X}_+$ but $\mathbf{y}\not\in\mathcal{X}_+$, then $N_{\textrm{den}}=0$ (see the notation of $\tilde{\mathcal{T}}^{(\tau)}_{T}$ for definition of $N_{\textrm{den}}$) as $\mathbf{y}$ will never be visited and thus $\tilde{\mathcal{T}}^{(\tau)}_{T,i}(\mathbf{x},\mathbf{y})=0=\mathcal{T}_i(\mathbf{x},\mathbf{y})$ also holds. Let us consider the entries in $\tilde{\mathcal{T}}^{(\tau)}_{T,i}(\mathbf{x},\mathbf{y})$ with $\mathbf{x},\mathbf{y}\in\mathcal{X}_+$ in the following.

Because $\mathbf{X}_{-i}$ is not updated at $i^{th}$ step of every sweep, we can replace $i-1$ in the definition of $N_{\textrm{den}}$ by $i$ and get
$$
N_{\textrm{den}} = \sum_{k=\tau}^{\tau+T-1} \mathbb{I}(\mathbf{X}_{-i}^{(kN + i)}=\mathbf{x}_{-i}).
$$
Notice that the set of times $\{t = kN + i : \tau \leq k \leq \tau + T - 1, \mathbf{X}_{-i}^{t}=\mathbf{x}_{-i})\}$, whose size is $N_{\textrm{den}}$, is a consecutive set of times when $w_{i, \mathbf{x}_{-i}}$ is updated. By Lemma \ref{lem:single_node}, we obtain a bound for the numerator
\begin{align}
& N_{\textrm{num}} \in [N_{\textrm{den}} \pi(X_i=y_i | \mathbf{x}_{-i})-1, N_{\textrm{den}} \pi(X_i=y_i | \mathbf{x}_{-i})+1] \Leftrightarrow \nonumber\\
& |P^{(\tau)}_{T, i}(X_i = y_i | \mathbf{x}_{-i}) - \pi(X_i=y_i | \mathbf{x}_{-i})| = |\frac{N_{\textrm{num}}}{N_{\textrm{den}}} - \pi(X_i=y_i | \mathbf{x}_{-i})| \leq \frac{1}{N_{\textrm{den}}}.  \label{eqn:num}
\end{align}
Also by Proposition \ref{prop:reachability}, we know every state in $\mathcal{X}_+$ is visited at a linear rate, there hence exist constants $l > 0$ and $B > 0$, such that the number of occurrence of any conditioning state $\mathbf{x}_{-i}$, $N_{\textrm{den}}$, is bounded by
\begin{equation}
N_{\textrm{den}} \geq \sum_{k=\tau}^{\tau+T-1} \mathbb{I}(\mathbf{X}^{(kN + i)}=\mathbf{x}) \geq lT-B \geq \frac{l}{2}T,\quad \forall T \geq \frac{2 B}{l}. \label{eqn:den}
\end{equation}
Combining equations (\ref{eqn:num}) and (\ref{eqn:den}), we obtain
\begin{align}
|P^{(\tau)}_{T, i}(X_i = y_i | \mathbf{x}_{-i}) - \pi(X_i=y_i | \mathbf{x}_{-i})| \leq \frac{2}{lT},\quad \forall T \geq \frac{2 B}{l}.
\end{align}
Since the matrix $\tilde{\mathcal{T}}^{(\tau)}_{T, i}$ and $\mathcal{T}_i$ differ only at those elements where $\mathbf{x}_{-i} = \mathbf{y}_{-i}$, we can bound the $L_1$ induced norm of the transposed matrix of their difference by
\begin{align}
\|(\tilde{\mathcal{T}}^{(\tau)}_{T, i} - \mathcal{T}_i)^T\|_1 &= \max_{\mathbf{x}} \sum_{\mathbf{y}} |\tilde{\mathcal{T}}^{(\tau)}_{T, i}(\mathbf{x},\mathbf{y}) - \mathcal{T}_i(\mathbf{x},\mathbf{y})| \nonumber\\
&= \max_{\mathbf{x}} \sum_{y_i} |P^{(\tau)}_{T, i}(X_i = y_i | \mathbf{x}_{-i}) - \pi(X_i=y_i | \mathbf{x}_{-i})| \nonumber\\
& \leq \frac{4}{lT},\quad \forall T \geq \frac{2 B}{l} \label{eqn:diff_Ti}
\end{align}
Observing that both $\tilde{\mathcal{T}}^{(\tau)}_{T}$ and $\mathcal{T}$ are multiplications of $N$ component transition matrices, and the transition matrices, $\tilde{\mathcal{T}}^{(\tau)}_{T}$ and $\mathcal{T}_i$, have a unit $L_1$ induced norm as:
\begin{align}
\|(\tilde{\mathcal{T}}^{(\tau)}_{T, i})^T\|_1 &= \max_{\mathbf{x}} \sum_{\mathbf{y}} |\tilde{\mathcal{T}}^{(\tau)}_{T, i}(\mathbf{x},\mathbf{y})| = \max_{\mathbf{x}} \sum_{\mathbf{y}} P^{(\tau)}_{T, i}(X_i = y_i | \mathbf{x}_{-i}) = 1 \label{eqn:unit_norm_tT} \\
\|(\mathcal{T}_i)^T\|_1 &= \max_{\mathbf{x}} \sum_{\mathbf{y}} |\mathcal{T}_i(\mathbf{x},\mathbf{y})| = \max_{\mathbf{x}} \sum_{\mathbf{y}} P(X_i = y_i | \mathbf{x}_{-i}) = 1 \label{eqn:unit_norm_T}
\end{align}
we can further bound the $L_1$ norm of the difference, $(\tilde{\mathcal{T}}^{(\tau)}_{T}-\mathcal{T})^T$. Let $P \in \mathbb{R}^N$ be any vector with nonzero norm. Using the triangular inequality, the difference of the resulting vectors after applying $\tilde{\mathcal{T}}^{(\tau)}_{T}$ and $\mathcal{T}$ is bounded by
\begin{align}
\|P(\tilde{\mathcal{T}}^{(\tau)}_{T} - \mathcal{T})\|_1 = & \|P\tilde{\mathcal{T}}^{(\tau)}_{T,1}\dots \tilde{\mathcal{T}}^{(\tau)}_{T,N} - P\mathcal{T}\dots \mathcal{T}_N\|_1 \nonumber \\
\leq & \|P\tilde{\mathcal{T}}^{(\tau)}_{T,1} \tilde{\mathcal{T}}^{(\tau)}_{T,2}\dots \tilde{\mathcal{T}}^{(\tau)}_{T,N} 
- P\mathcal{T}_1 \tilde{\mathcal{T}}^{(\tau)}_{T,2}\dots \tilde{\mathcal{T}}^{(\tau)}_{T,N}\|_1 + \nonumber\\
&\|P\mathcal{T}_1 \tilde{\mathcal{T}}^{(\tau)}_{T,2} \tilde{\mathcal{T}}^{(\tau)}_{T,3}\dots \tilde{\mathcal{T}}^{(\tau)}_{T,N} 
- P\mathcal{T}_1 \mathcal{T}_2 \tilde{\mathcal{T}}^{(\tau)}_{T,3} \dots \tilde{\mathcal{T}}^{(\tau)}_{T,N}\|_1 + \nonumber\\
& \dots \nonumber\\
& \|P\mathcal{T}_1 \dots \mathcal{T}_{N-1} \tilde{\mathcal{T}}^{(\tau)}_{T,N} 
- P\mathcal{T}_1 \dots \mathcal{T}_{N-1} \mathcal{T}_N \|_1
\end{align}
where the $i$'th term is
\begin{align}
\|P\mathcal{T}_1 \dots \mathcal{T}_{i-1} (\tilde{\mathcal{T}}^{(\tau)}_{T,i} - \mathcal{T}_i) \tilde{\mathcal{T}}^{(\tau)}_{T,i+1}\dots \tilde{\mathcal{T}}^{(\tau)}_{T,N} \|_1 
& \leq \|P\mathcal{T}_1 \dots \mathcal{T}_{i-1} (\tilde{\mathcal{T}}^{(\tau)}_{T,i} - \mathcal{T}_i)\|_1 & \text{(Unit $L_1$ norm, Eqn.~\ref{eqn:unit_norm_tT})} \nonumber\\
& \leq \|P\mathcal{T}_1 \dots \mathcal{T}_{i-1}\|_1 \frac{4}{lT} & \text{(Eqn.~\ref{eqn:diff_Ti})} \nonumber\\
& \leq \|P\|_1 \frac{4}{lT} & \text{(Unit $L_1$ norm, Eqn.~\ref{eqn:unit_norm_T})}
\end{align}
Consequently, we get the $L_1$ induced norm of $(\tilde{\mathcal{T}}^{(\tau)}_{T} - \mathcal{T})^T$ as
\begin{equation}
\|(\tilde{\mathcal{T}}^{(\tau)}_{T} - \mathcal{T})^T\| = \max_P \frac{\|P(\tilde{\mathcal{T}}^{(\tau)}_{T} - \mathcal{T})\|_1}{\|P\|_1} \leq \frac{4N}{lT},\quad \forall T \geq \frac{2 B}{l},
\end{equation}
\end{proof}

\subsection{Proof of Theorem \ref{thm:convergenceHerdedGibbs}}
When we initialize the herded Gibbs and regular Gibbs with the same distribution (see Figure \ref{fig:distr_dist}), since the transition kernel of herded Gibbs is an approximation to regular Gibbs and the distribution of regular Gibbs converges to the invariant distribution, we expect that herded Gibbs also approaches the invariant distribution.


\begin{proof}[Proof of Theorem \ref{thm:convergenceHerdedGibbs}]
Construct an auxiliary regular Gibbs sampling Markov chain initialized with $\pi^{(0)}(\mathbf{X})=P^{(\tau)}_T(\mathbf{X})$ and the same scanning order as herded Gibbs. As $\eta<1$, the Gibbs Markov chain has uniform geometric convergence rate as shown in Equation (\ref{eqn:dvgibbs_inv}). 

Also, the Gibbs Markov chain must be irreducible due to $\eta<1$ and therefore Proposition \ref{prop:err_onestep} applies here. We can bound the distance between the distributions of herded Gibbs after one sweep of all variables, $P_T^{(\tau+1)}$, and the distribution after one sweep of regular Gibbs sampling, $\pi^{(1)}$ by
\begin{align}
& d_v(P_T^{(\tau+1)} - \pi^{(1)}) = d_v(\pi^{(0)}(\tilde{\mathcal{T}}^{(\tau)}_{T} - \mathcal{T})) = \frac{1}{2}\|\pi^{(0)}(\tilde{\mathcal{T}}^{(\tau)}_{T} - \mathcal{T})\|_1 \nonumber\\
& \leq \frac{2N}{lT} \|\pi^{(0)}\|_1 = \frac{2N}{lT},\quad \forall T \geq T^*, \tau \geq 0. \label{eqn:dvgibbs_herd}
\end{align}

Now we study the change of discrepancy between $P_T^{(\tau)}$ and $\pi$ as a function as $\tau$.

Applying the triangle inequality of $d_v$:
\begin{align}
& d_v(P_T^{(\tau+1)} - \pi) = d_v(P_T^{(\tau+1)} - \pi^{(1)} + \pi^{(1)} - \pi) \leq d_v(P_T^{(\tau+1)} - \pi^{(1)}) + d_v(\pi^{(1)} - \pi) \nonumber\\
& \leq \frac{2N}{lT} + \eta d_v(P_T^{(\tau)} - \pi),\quad \forall T \geq T^*, \tau \geq 0. \label{eqn:tau+1}
\end{align}
The last inequality follows Equations (\ref{eqn:dvgibbs_inv}) and (\ref{eqn:dvgibbs_herd}). When the sample distribution is outside a neighborhood of $\pi$, $\mathcal{B}_{\epsilon_1}(\pi)$, with $\epsilon_1=\frac{4N}{(1 - \eta)lT}$, i.e.
\begin{equation}
d_v(P_T^{(\tau)} - \pi) \geq \frac{4N}{(1 - \eta)lT}, \label{eqn:case1}
\end{equation}
we get a geometric convergence rate toward the invariant distribution by combining the two equations above:
\begin{equation}
d_v(P_T^{(\tau+1)} - \pi) \leq \frac{1-\eta}{2}d_v(P_T^{(\tau)} - \pi) + \eta d_v(P_T^{(\tau)} - \pi) = \frac{1+\eta}{2}d_v(P_T^{(\tau)} - \pi). \label{eqn:geo_herd}
\end{equation}
So starting from $\tau=0$, we have a burn-in period for herded Gibbs to enter $\mathcal{B}_{\epsilon_1}(\pi)$ in a finite number of rounds. Denote the first time it enters the neighborhood by $\tau'$. According to the geometric convergence rate in Equations \ref{eqn:geo_herd} and $d_v(P_T^{(0)} - \pi) \leq 1$
\begin{equation}
\tau' \leq \left\lceil \log_{\frac{1+\eta}{2}}(\frac{\epsilon_1}{d_v(P_T^{(0)} - \pi)}) \right\rceil \leq \left\lceil \log_{\frac{1+\eta}{2}}(\epsilon_1) \right\rceil = \lceil \tau^*(T) \rceil. \label{eqn:tau_T}
\end{equation}
After that burn-in period, the herded Gibbs sampler will stay within a smaller neighborhood, $\mathcal{B}_{\epsilon_2}(\pi)$, with $\epsilon_2=\frac{1+\eta}{1-\eta}\frac{2N}{l T}$, i.e.
\begin{equation}
d_v(P_T^{(\tau)} - \pi) \leq \frac{1+\eta}{1-\eta}\frac{2N}{l T},\quad \forall \tau > \tau'. \label{eqn:bound}
\end{equation}
This is proved by induction:
\begin{enumerate}
\item Equation (\ref{eqn:bound}) holds at $\tau=\tau'+1$. This is because $P_T^{(\tau')}\in \mathcal{B}_{\epsilon_1}(\pi)$ and following Eqn.~\ref{eqn:tau+1} we get
\begin{equation}
d_v(P_T^{(\tau'+1)} - \pi) \leq \frac{2N}{lT} + \eta \epsilon_1 = \epsilon_2
\end{equation}
\item For any $\tau \geq \tau'+2$, assume $P_T^{(\tau-1)} \in \mathcal{B}_{\epsilon_2}(\pi)$. Since $\epsilon_2 < \epsilon_1$, $P_T^{(\tau-1)}$ is also in the ball $\mathcal{B}_{\epsilon_1}(\pi)$. We can apply the same computation as when $\tau = \tau'+1$ to prove $d_v(P_T^{(\tau)} - \pi) \leq \epsilon_2$.
So inequality (\ref{eqn:bound}) is always satisfied by induction.
\end{enumerate}
Consequently, Theorem \ref{thm:convergenceHerdedGibbs} is proved when combining (\ref{eqn:bound}) with the inequality $\tau' \leq \lceil \tau^*(T) \rceil$ in Equation( \ref{eqn:tau_T}).
\end{proof}

\begin{remark}
Similarly to the regular Gibbs sampler, the herded Gibbs sampler also has a burn-in period with geometric convergence rate. After that, the distribution discrepancy is in the order of $O(1/T)$, which is faster than the regular Gibbs sampler. Notice that the length of the burn-in period depends on $T$, specifically as a function of $\log(T)$.
\end{remark}
\begin{remark}
Irrationality is not required to prove the convergence on a fully-connected graph.
\end{remark}

\begin{corollary} \label{cor:logT_T_convergence}
When the conditions of Theorem \ref{thm:convergenceHerdedGibbs} hold, and we start collecting samples at the end of every sweep from the beginning, the error of the sample distribution is bounded by:
\begin{equation}
d_v(P_T^{(\tau=0)} - \pi) \leq \frac{\lambda + \tau^*(T)}{T} = O(\frac{\log(T)}{T}),\quad \forall T\geq T^*+\tau^*(T^*)
\end{equation}
\end{corollary}
\begin{proof}
Since $\tau^*(T)$ is a monotonically increasing function of $T$, for any $T\geq T^*+\tau^*(T^*)$, we can find a number $t$ so that
$$
T = t+\tau^*(t), t \geq T^*.
$$
Partition the sample sequence $S_{0,T}=\{\mathbf{X}^{(kN)}:0\leq k <T\}$ into two parts: the burn-in period $S_{0,\tau^*(t)}$ and the stable period $S_{\tau^*(t),T}$. The discrepancy in the burn-in period is bounded by $1$ and according to Theorem \ref{thm:convergenceHerdedGibbs}, the discrepancy in the stable period is bounded by 
$$
d_v(\tilde{P}(S_{t,T}) - \pi) \leq \frac{\lambda}{t}.
$$
Hence, the discrepancy of the whole set $S_{0,T}$ is bounded by
\begin{align}
& d_v(\tilde{P}(S_{0,T}) - \pi) = d_v\left(\frac{\tau^*(t)}{T}\tilde{P}(S_{0,\tau^*(t)})+\frac{t}{T}\tilde{P}(S_{\tau^*(t),T}) - \pi\right) \nonumber\\
& \leq d_v\left(\frac{\tau^*(t)}{T}(\tilde{P}(S_{0,\tau^*(t)}) - \pi\right) + d_v\left(\frac{t}{T}(\tilde{P}(S_{\tau^*(t),T}) - \pi\right) \nonumber\\
& \leq \frac{\tau^*(t)}{T} d_v(\tilde{P}(S_{0,\tau^*(t)}) - \pi) + \frac{t}{T} d_v(\tilde{P}(S_{\tau^*(t),T}) - \pi) \nonumber\\
& \leq \frac{\tau^*(t)}{T} \cdot 1 + \frac{t}{T} \frac{\lambda}{t} \leq \frac{\tau^*(T) + \lambda}{T}.
\end{align}
\end{proof}

\end{document}